\documentclass[11pt]{article}

\PassOptionsToPackage{table,x11names,dvipsnames}{xcolor}
\usepackage[final]{acl}
\usepackage{times}
\usepackage{latexsym}
\usepackage[T1]{fontenc}
\usepackage[utf8]{inputenc}
\usepackage{microtype}
\usepackage{inconsolata}
\usepackage{graphicx}
\usepackage{subcaption}
\usepackage{amsmath}
\usepackage{amsfonts}
\usepackage{amssymb}
\usepackage{amsthm}
\usepackage{enumitem}
\newtheorem{proposition}{Proposition}
\usepackage{bm}
\newcommand{\mat}[1]{\bm{#1}}
\DeclareRobustCommand{\ours}{\texorpdfstring{\textsc{FraQ}}{FraQ}}
\usepackage{booktabs}
\usepackage{tabularx}
\newcolumntype{Y}{>{\raggedright\arraybackslash}X}
\newcolumntype{Z}{>{\centering\arraybackslash}X}
\usepackage{multirow}
\usepackage{placeins}
\usepackage{pifont}
\usepackage{tikz}
\usepackage[ruled,vlined]{algorithm2e}

\makeatletter
\def\input@path{{./}}
\makeatother
\graphicspath{{./}}

\title{\ours{}: Efficient Coordinate-Space Recompression for Federated Low-Rank Adaptation}

\author{
  Shenghui Li$^{1}$ and Thiemo Voigt$^{1,2}$ \\
  $^{1}$Uppsala University, Uppsala, Sweden \\
  $^{2}$Research Institutes of Sweden, Stockholm, Sweden \\
  \texttt{\{shenghui.li,thiemo.voigt\}@angstrom.uu.se}
}

\begin{document}
\maketitle

\begin{abstract}
	Federated fine-tuning with Low-Rank Adaptation (LoRA) enables
	collaborative and efficient adaptation of Large Language Models
	(LLMs) without centralizing private data. However, its two-factor parameterization introduces an aggregation mismatch across clients, as naive factor-wise averaging does not recover the average of the induced updates. Forming the exact aggregate in the full weight space and recompressing it avoids this mismatch, but decomposing the resulting full-size matrix is computationally expensive and memory-intensive. We propose \ours{}, an efficient coordinate-space recompression method for federated LoRA. Starting from stacked factors that represent the exact aggregate, \ours{} expresses the aggregate as the product of an orthonormal basis and a compact coordinate matrix. It then recovers the singular spectrum from a small Gram matrix of this coordinate representation and selects the smallest rank satisfying a prescribed energy threshold. The selected coordinate subspace is mapped back via the basis to construct the global adapter. Experiments across text classification and commonsense reasoning benchmarks demonstrate that \ours{} achieves accuracy close to that	of uncompressed baselines while substantially reducing downlink
	communication and incurring low server-side recompression overhead. 
\end{abstract}

\section{Introduction}
\label{sec:intro}

\begin{figure}[!t]
    \centering
    \includegraphics[width=1.0\columnwidth]{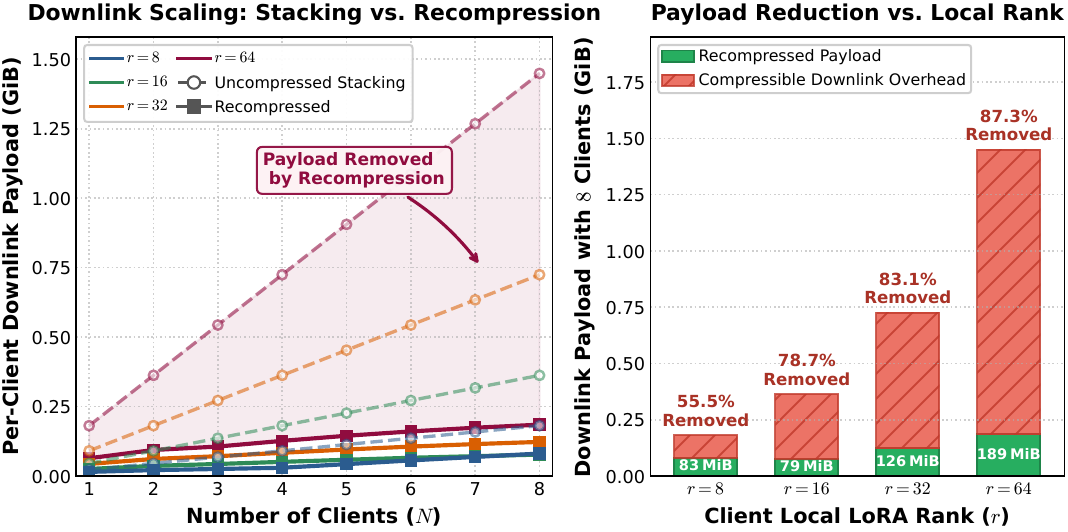}
\caption{Exact aggregation via unrecompressed stacking rapidly increases
	per-client downlink with client count and local rank. On
	LLaMA-3.2-3B, recompression with at least 90\% spectral-energy
	retention cuts this payload by 55.5--87.3\%.}
    \label{fig:downlink-motivation}
\end{figure}

Large Language Models (LLMs) have demonstrated remarkable capabilities across a wide range of tasks~\citep{brown2020language,touvron2023llama,grattafiori2024llama3herdmodels,qwen2.5}.  While fine-tuning further improves the adaptability of LLMs to downstream domains~\citep{howard-ruder-2018-universal}, fully updating their massive parameter sets remains computationally expensive and susceptible to overfitting. Parameter-Efficient Fine-Tuning (PEFT) methods alleviate these limitations by adapting LLMs with only a small number of trainable parameters, with Low-Rank Adaptation (LoRA)~\citep{hu2022lora} being a representative approach that introduces trainable low-rank updates to existing weight matrices.

When the data required for fine-tuning LLMs are distributed across institutions or edge devices, directly centralizing them is often infeasible due to privacy or regulatory constraints. Federated Learning (FL)~\citep{pmlr-v54-mcmahan17a} provides a natural paradigm for collaborative learning without sharing local data. When combined with LoRA, FL allows clients to train and communicate
only compact low-rank adapters rather than full model parameters,
substantially reducing communication overhead. %
However, LoRA also makes aggregation more subtle: conventional FedAvg-style aggregation~\cite{10447454} averages the two LoRA factor matrices independently, although the product of averaged factors generally differs from the average of the corresponding low-rank updates. This mismatch produces an inexact global update and can degrade
fine-tuning performance~\citep{sun2024improving,
	singhal-etal-2025-fedex}.

To address this mismatch, existing methods either constrain the local LoRA parameterization~\citep{sun2024improving} or move to update-level aggregation through residual merging~\citep{singhal-etal-2025-fedex} or factor stacking~\citep{wang2024florafederatedfinetuninglarge}.
The former limits local adaptation, whereas update-level methods
better preserve the aggregate without necessarily producing a
compact rank-adaptive global adapter. %
In particular, unrecompressed stacking exactly represents the aggregate, yet its adapter width grows with the sum of participating client ranks~\citep{wang2024florafederatedfinetuninglarge,singhal-etal-2025-fedex}, increasing the per-client downlink payload and local computation cost. Intuitively, the stacked width can be much larger than the number of directions needed to capture most of the aggregate's spectral energy~\citep{yan2026fedmomentumpreservingloratraining,ramesh2026florist}. The server can therefore recompress the aggregate into a lower-rank adapter before broadcasting it. As illustrated in Figure~\ref{fig:downlink-motivation}, recompression removes much of the downlink overhead while retaining most of the spectral energy.

However, conventional approaches materialize the dense aggregate in
the full weight space and apply truncated SVD~\citep{
	bai2024federated,singhal-etal-2025-fedex}, incurring substantial
server-side memory and computation. Recent compact-space variants
avoid dense reconstruction, yet still require repeated truncated
SVDs~\citep{panariello2025accurate,ramesh2026florist} or two reduced
QR factorizations followed by an SVD of the compact
core~\citep{ramesh2026spectraltransformationlayerwiseglobal}.

We propose \ours{}, a direction-adaptive, one-sided coordinate-space
recompression framework for federated LoRA. \ours{} confines
orthogonalization to the smaller weight dimension, yielding an
implicit basis and exact coordinates. It then recovers the spectrum
through a compact Gram eigenproblem, selects the global rank using an
energy threshold, and maps only the retained directions back through
the implicit basis to form an optimal rank-$p$ adapter. Our main
contributions are summarized as follows:

\begin{itemize}
	\item We establish an isometric coordinate formulation for federated LoRA recompression, proving that compact coordinates of the exact stacked aggregate preserve its nonzero singular spectrum and %
	low-rank approximation geometry. This guarantees full-space	rank-$p$ optimality after coordinate-space truncation.
    \item We design \ours{}, a direction-adaptive, one-sided algorithm that
orients the factorization along the smaller weight dimension, constructs
an implicit basis with a single reduced QR, and performs rank-adaptive
recompression in compact coordinate space. \ours{} thereby avoids materializing and decomposing the full-size aggregate while remaining equivalent to full-space recompression.
	\item Extensive experiments across multiple tasks under varying
	data heterogeneity and both uniform and heterogeneous client ranks
	show that \ours{} maintains accuracy close to that of uncompressed
	aggregation while substantially reducing downlink communication
	and incurring low server-side recompression overhead.
\end{itemize}

\section{Preliminaries}
\label{sec:prelim}

\begin{figure*}[t!]
	\centering
	\includegraphics[width=0.99\textwidth]{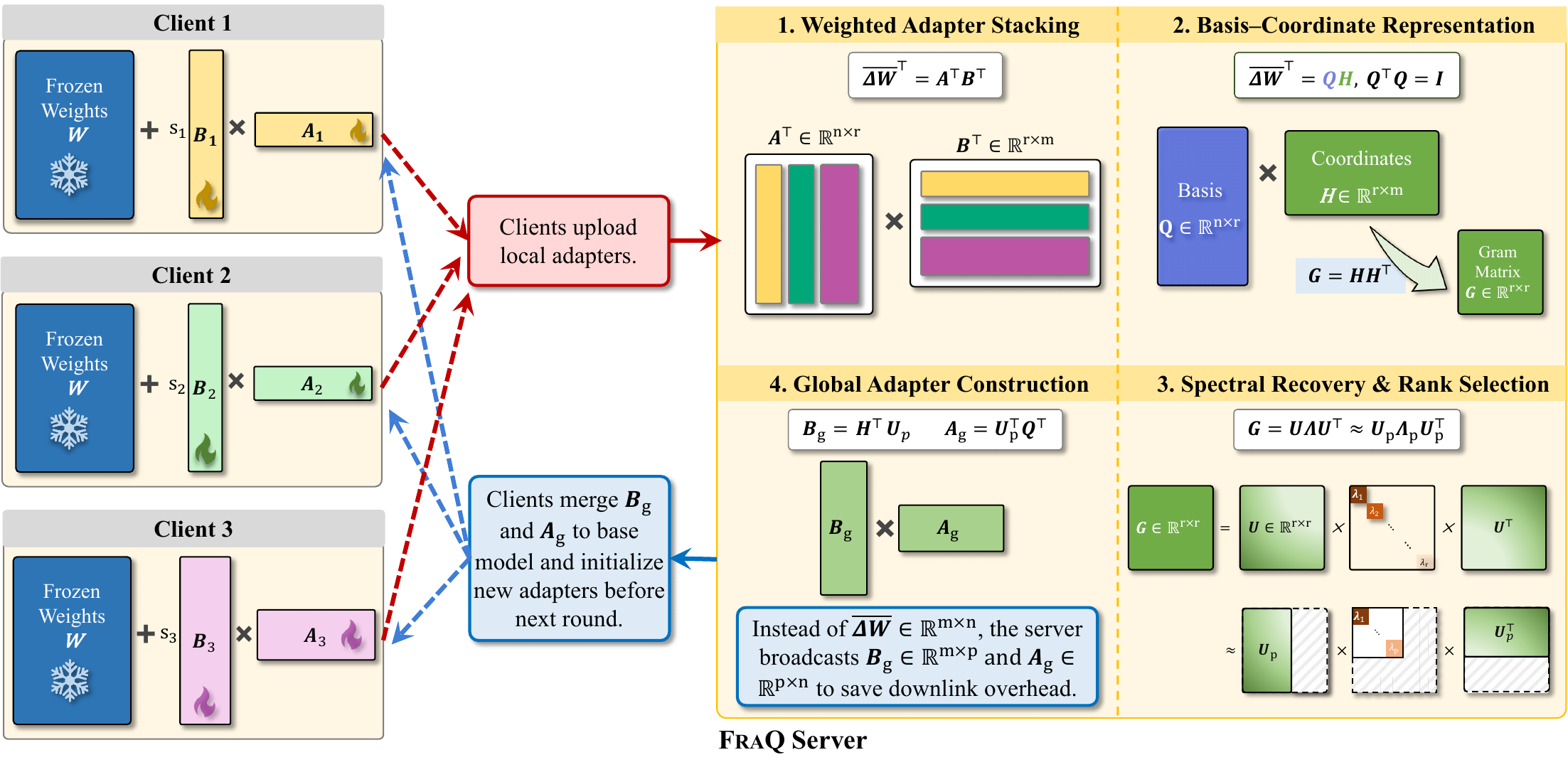}
	\caption{
		Overview of \ours{}, illustrated for the $m>n$ orientation.
		Each client initializes and trains a LoRA adapter on the current global model and uploads its factors to the server.
		The server (1) forms weighted factor stacks whose product equals the exact update-level aggregate; (2) constructs an orthonormal basis and a coordinate representation without materializing the full-size aggregate; (3) recovers the spectrum from a small Gram matrix and selects the global rank using an energy-retention threshold; and (4) combines the retained coordinate-space approximation with the basis to construct compact global LoRA factors for broadcast. Clients merge the broadcast global adapter into their cached global weights and reinitialize fresh local adapters for the next round.
		}
	\label{fig:workflow}
\end{figure*}

\paragraph{Fine-tuning with LoRA.}
LoRA~\cite{hu2022lora} represents the update of a pretrained weight matrix through a low-rank factorization. The fine-tuned weights $\mat{W}'$ are expressed as the sum of the frozen pretrained weights $\mat{W}_0$ and a low-rank update $\Delta\mat{W}$:
\begin{equation}
    \mat{W}'
    =
    \mat{W}_0+\Delta\mat{W}
    =
    \mat{W}_0+s\mat{B}\mat{A},
    \label{eq:lora}
\end{equation}
where $\mat{W}_0,\mat{W}'\in\mathbb{R}^{m\times n}$ are the pretrained and fine-tuned weights, $\mat{B}\in\mathbb{R}^{m\times\rho}$ and $\mat{A}\in\mathbb{R}^{\rho\times n}$ are the trainable low-rank factors, and $s=\alpha_{\mathrm{LoRA}}/\rho$ is the LoRA scaling factor. Thus, the effective weight update is $\Delta\mat{W}=s\mat{B}\mat{A}$. %
Instead of updating $\mat{W}_0$ directly, LoRA trains only the small factors $\mat{B}$ and $\mat{A}$, reducing the number of trainable parameters from $mn$ to $\rho(m+n)$.

\paragraph{Federated fine-tuning.}
Federated learning~\cite{pmlr-v54-mcmahan17a} trains a shared model across a set $\mathcal{C}$ of $K$ clients holding private data, coordinated by a central server over $T$ communication rounds. Under FedAvg, the server aggregates local updates by a data-weighted average; with $n_k$ samples at client $k$, $N=\sum_k n_k$, and weight $\alpha_k=n_k/N$, the global update is
\begin{equation}
    \overline{\Delta\mat{W}}
    = \sum_{k=1}^{K}\alpha_k\,\Delta\mat{W}_k ,
    \label{eq:fedavg}
\end{equation}
where $\Delta\mat{W}_k$ is the local update from client $k$. FedIT~\cite{10447454} extends this scheme to LoRA: each client $k$ fine-tunes local adapters $(\mat{B}_k,\mat{A}_k)$. FedIT assumes homogeneous LoRA configurations, so all clients use a common LoRA scale $s$. The server averages the factors independently,
\begin{equation}
    \overline{\mat{B}} = \sum_{k=1}^{K}\alpha_k\,\mat{B}_k ,
    \qquad
    \overline{\mat{A}} = \sum_{k=1}^{K}\alpha_k\,\mat{A}_k .
    \label{eq:fedit}
\end{equation}

\paragraph{Inexactness of LoRA aggregation.}
The difficulty is that aggregating adapters is not equivalent to aggregating their updates. Because each client's effective contribution is $s\mat{B}_k\mat{A}_k$ rather than the individual factors, independently averaging $\mat{B}_k$ and $\mat{A}_k$ does not recover the exact aggregated update:
\begin{equation}
    \underbrace{s\overline{\mat{B}}\,\overline{\mat{A}}}_{\text{FedIT: inexact}}
    = s\!\sum_{i,j}\alpha_i\alpha_j\,\mat{B}_i\mat{A}_j
    \!\neq\!
    \underbrace{s\sum_k\alpha_k\,\mat{B}_k\mat{A}_k}
    _{\overline{\Delta\mat{W}}\text{: exact}} .
    \label{eq:inexact}
\end{equation}
The discrepancy stems from the cross terms $\mat{B}_i\mat{A}_j$ with $i\neq j$, i.e., {the average of the products is not equal to the product of the averages}. A naive remedy is to materialize the exact aggregate $\overline{\Delta\mat{W}}\in\mathbb{R}^{m\times n}$ in Eq.~\eqref{eq:inexact} and re-decompose it with a truncated SVD, but this discards the low-rank structure that makes LoRA efficient and incurs a costly full-matrix decomposition on the server. Our method computes the rank-truncated form of the exact aggregate $\overline{\Delta\mat{W}}$ directly in a compact coordinate space, avoiding both the cross-term noise of Eq.~\eqref{eq:inexact} and the cost of full-matrix SVD.

\section{Proposed Method}
\label{sec:proposed}

We introduce \ours{}, an efficient method for consolidating client LoRA updates into a compact global adapter. As illustrated in Figure~\ref{fig:workflow}, \ours{} first forms weighted factor stacks whose product equals the exact update-level aggregate and represents this aggregate through an orthonormal basis and a compact coordinate matrix. It then recovers the spectrum from a small Gram matrix formed from the coordinates, uses an energy-retention threshold to select the broadcast rank, and maps the retained directions through the basis to construct the LoRA factors broadcast to clients.

\subsection{Problem Formulation}
\label{subsec:problem-setup}

At the beginning of round $t$, all clients initialize from the current global weights $\mat{W}_t$, which incorporate all previously merged updates into the pretrained model $\mat{W}_0$. Client $k$ then fine-tunes a rank $r_k$ low-rank adapter $(\mat{B}_k,\mat{A}_k)$, with $\mat{B}_k\in\mathbb{R}^{m\times r_k}$ and $\mat{A}_k\in\mathbb{R}^{r_k\times n}$, on top of $\mat{W}_t$, yielding the effective local update $\Delta\mat{W}_k=s_k\mat{B}_k\mat{A}_k$, where $s_k=\alpha_{\mathrm{LoRA},k}/r_k$.
For clarity, we omit round indices on adapter factors throughout
Sections~\ref{subsec:problem-setup}--\ref{sec:fraq-recompression};
Figure~\ref{fig:workflow} follows the same convention.
With aggregation weights $\alpha_k=n_k/N$, the server seeks the weighted aggregate of these updates,
\begin{equation}
    \overline{\Delta \mat{W}}
    =
    \sum_{k=1}^{K}\alpha_k \Delta\mat{W}_k
    =
    \sum_{k=1}^{K}\alpha_k s_k\mat{B}_k \mat{A}_k .
    \label{eq:exact-aggregate}
\end{equation}
Averaging $\mat{B}_k$ and $\mat{A}_k$ separately introduces the cross-term noise of Eq.~\eqref{eq:inexact}~\citep{sun2024improving}. Instead, following the exact weighted stacking scheme of FLoRA~\citep{wang2024florafederatedfinetuninglarge} and FLoRIST~\citep{ramesh2026florist}, the server constructs
\begin{align}
    \mat{B} &=
    \left[
    \sqrt{\alpha_1}\mat{B}_1,\ldots,\sqrt{\alpha_K}\mat{B}_K
    \right]
    \in\mathbb{R}^{m\times r},
    \label{eq:weighted-stacking}\\[-0.2ex]
    \mat{A} &=
    \left[
    \sqrt{\alpha_1}s_1\mat{A}_1^{\top},\ldots,
    \sqrt{\alpha_K}s_K\mat{A}_K^{\top}
    \right]^{\top}
    \in\mathbb{R}^{r\times n}.
    \notag
\end{align}
where $r=\sum_{k=1}^{K}r_k$ is the total client rank. The client-specific
LoRA scale is incorporated into the stacked $\mat{A}$ factor, matching the
implementation. Consequently,
$\mat{B}\mat{A}=\sum_k\alpha_k s_k\mat{B}_k\mat{A}_k
=\overline{\Delta\mat{W}}$. The aggregate is therefore represented
{exactly} as
\begin{equation}
    \overline{\Delta\mat{W}}=\mat{B}\mat{A},
    \qquad
    \operatorname{rank}\bigl(\overline{\Delta\mat{W}}\bigr)\leq r,
    \label{eq:stack-exact}
\end{equation}
with heterogeneous client ranks $r_k$ accommodated naturally by the concatenation. This exactness, however, comes at a cost: stacking inflates the individual rank-$r_k$ adapters into an aggregate of rank up to the {total} client rank $r$, which grows with the number of participating clients. To keep the downlink compact, the server must instead broadcast a rank-$p$ adapter $\mat{B}_g\in\mathbb{R}^{m\times p}$ and $\mat{A}_g\in\mathbb{R}^{p\times n}$, where $p\ll r$.
The problem to solve is therefore the \emph{recompression} of the exact aggregate: find the rank-$p$ adapter closest to $\overline{\Delta\mat{W}}$ in Frobenius norm, $\min_{\mat{B}_g,\mat{A}_g}\lVert \overline{\Delta\mat{W}} - \mat{B}_g\mat{A}_g \rVert_F$, or equivalently
\begin{equation}
    \min_{\operatorname{rank}(\mat{Z})\leq p}
    \bigl\lVert \overline{\Delta\mat{W}} - \mat{Z} \bigr\rVert_F .
    \label{eq:recompression}
\end{equation}

\subsection{Isometric Coordinate Representation}
For any compatible matrix product $\mat{M}=\mat{X}\mat{Y}$,
we have $\operatorname{col}(\mat{M})
\subseteq \operatorname{col}(\mat{X})$. If $\mat{X}$ has $r$ columns, then
$\operatorname{rank}(\mat{M})
\le \operatorname{rank}(\mat{X}) \le r$.
Thus, although $\mat{M}$ may be full-sized, its columns lie in a
subspace of dimension at most $r$. Representing $\mat{M}$ in an
orthonormal basis whose span contains this subspace is therefore exact.
The following proposition shows that the resulting basis--coordinate
mapping is isometric, allowing optimal low-rank recompression to be
performed equivalently in the coordinate and original matrix spaces.

\begin{proposition}[Isometric coordinate equivalence]
	\label{prop:core-equivalence}
	Let $\mat{M}=\mat{X}\mat{Y}$, where
	$\mat{X}\in\mathbb{R}^{d\times r}$ and
	$\mat{Y}\in\mathbb{R}^{r\times D}$, with $r\le d$.
	Let $\mat{Q}\in\mathbb{R}^{d\times r}$ have orthonormal columns
	whose span contains $\operatorname{col}(\mat{X})$, i.e.,
	$
	\mat{Q}^{\top}\mat{Q}=\mat{I}_r
	$ and $
	\operatorname{col}(\mat{X})
	\subseteq
	\operatorname{col}(\mat{Q}).
	$
	Define
		\begin{equation}
			\begin{aligned}
				\mat{C}
				&=\mat{Q}^{\top}\mat{X}
				\in\mathbb{R}^{r\times r},\\
				\mat{H}
				&=\mat{C}\mat{Y}
				=\mat{Q}^{\top}\mat{M}
				\in\mathbb{R}^{r\times D}.
			\end{aligned}
			\label{eq:coordinate-representation}
		\end{equation}
		Then $\mat{M}=\mat{Q}\mat{H}$. Moreover, for any
		$\widehat{\mat{H}}\in\mathbb{R}^{r\times D}$,
		\begin{equation}
			\left\|
			\mat{M}-\mat{Q}\widehat{\mat{H}}
			\right\|_F
			=
			\left\|
			\mat{H}-\widehat{\mat{H}}
			\right\|_F .
			\label{eq:coordinate-isometry}
		\end{equation}
	Thus, the columns of $\mat{H}$ are the exact coordinates of the
	corresponding columns of $\mat{M}$ in the basis $\mat{Q}$, and
	$\mat{H}$ and $\mat{M}$ have the same nonzero singular values.
	Consequently, if $\mat{H}_p$ is a best rank-$p$ approximation of
	$\mat{H}$, then
	\begin{equation}
		\mat{Q}\mat{H}_p
		\in
		\operatorname*{arg\,min}_{
			\substack{
				\widetilde{\mat{M}}\in\mathbb{R}^{d\times D}\\
				\operatorname{rank}(\widetilde{\mat{M}})\le p
			}
		}
		\left\|
		\mat{M}-\widetilde{\mat{M}}
		\right\|_F .
		\label{eq:core-equivalence}
	\end{equation}
\end{proposition}

Proposition~\ref{prop:core-equivalence} shows that optimal rank-$p$
recompression can be performed on $\mat{H}$ in the $r$-dimensional
coordinate space induced by $\mat{Q}$, without materializing the full
matrix $\mat{M}$. The proposition is stated under the assumption
$r\le d$, which holds in all federated LLM fine-tuning settings evaluated in this work. The extension to
$r>d$ is provided in Appendix~\ref{app:general-complexity}, and the proof of
the proposition is provided in Appendix~\ref{app:proof}.

\subsection{\ours{} Recompression}
\label{sec:fraq-recompression}

Given the weighted factor stacks in Eq.~\eqref{eq:weighted-stacking},
\ours{} instantiates Proposition~\ref{prop:core-equivalence} to recompress
the exact aggregate. Because transposition preserves singular values, rank,
and the Frobenius norm, \ours{} can operate on either the aggregate or its
transpose. We therefore orient the factorization along the smaller weight
dimension:

\begingroup
\setlength{\abovedisplayskip}{5pt plus 1pt minus 2pt}
\setlength{\belowdisplayskip}{5pt plus 1pt minus 2pt}
\setlength{\abovedisplayshortskip}{0pt plus 1pt}
\setlength{\belowdisplayshortskip}{3pt plus 1pt minus 1pt}
\begin{equation}
    (\mat{X},\mat{Y})
    =
    \begin{cases}
        (\mat{B},\mat{A}), & m\le n,\\
        (\mat{A}^{\top},\mat{B}^{\top}), & m>n.
    \end{cases}
    \label{eq:orientation}
\end{equation}
\endgroup
Accordingly, $\mat{M}=\mat{X}\mat{Y}$ equals
$\overline{\Delta\mat{W}}$ when $m\le n$ and
$\overline{\Delta\mat{W}}^{\top}$ when $m>n$. In either case,
$\mat{X}\in\mathbb{R}^{d\times r}$ and
$\mat{Y}\in\mathbb{R}^{r\times D}$, where
$d=\min(m,n)$ and $D=\max(m,n)$.

\paragraph{Basis-coordinate representation.}
To obtain the orthonormal basis required by
Proposition~\ref{prop:core-equivalence}, \ours{} factorizes the
tall-and-skinny left factor
$\mat{X}\in\mathbb{R}^{d\times r}$.
We use a reduced Householder QR factorization as the default exact
basis construction, obtaining
$\mat{R}\in\mathbb{R}^{r\times r}$ and an implicit representation of
$\mat{Q}\in\mathbb{R}^{d\times r}$ such that
$\mat{X}=\mat{Q}\mat{R}$ and
$\mat{Q}^{\top}\mat{Q}=\mat{I}_r$.
Because
$
\operatorname{col}(\mat{M})
\subseteq
\operatorname{col}(\mat{X})
\subseteq
\operatorname{col}(\mat{Q}),
$
the columns of $\mat{M}$ are represented exactly in the basis
$\mat{Q}$. Substituting $\mat{X}=\mat{Q}\mat{R}$ into
$\mat{M}=\mat{X}\mat{Y}$ and defining
$\mat{H}=\mat{R}\mat{Y}$ gives
$\mat{M}=\mat{Q}\mat{H}$.
Whereas existing approaches to LoRA merging and recompression, such as
Core Space Merging~\citep{panariello2025accurate} and
FLoRIST~\citep{ramesh2026florist}, construct left and right bases
using a pair of thin SVDs, \ours{} obtains the required basis from a
single reduced Householder QR factorization of $\mat{X}$ and incorporates
$\mat{Y}$ directly through $\mat{R}\mat{Y}$.

\paragraph{Spectral recovery and rank selection.}
Since Proposition~\ref{prop:core-equivalence} establishes that
$\mat{H}$ and $\mat{M}$ share the same nonzero singular values,
\ours{} performs spectral analysis entirely in the coordinate space.
Rather than computing an SVD of
$\mat{H}\in\mathbb{R}^{r\times D}$, it forms the compact left Gram
matrix and computes its symmetric eigendecomposition:
\begin{equation}
	\mat{G}
	=
	\mat{H}\mat{H}^{\top}
	=
	\mat{U}\mat{\Lambda}\mat{U}^{\top}
	\in\mathbb{R}^{r\times r}.
	\label{eq:gram-eigendecomposition}
\end{equation}
Here $\mat{U}^{\top}\mat{U}=\mat{I}_r$ and
$\mat{\Lambda}=\operatorname{diag}(\lambda_1,\ldots,\lambda_r)$,
with $\lambda_1\ge\cdots\ge\lambda_r\ge0$.
Because $\mat{G}=\mat{H}\mat{H}^{\top}$, the columns of $\mat{U}$
are the left singular vectors of $\mat{H}$, and
$
\lambda_i
=
\sigma_i(\mat{H})^2
=
\sigma_i(\mat{M})^2.
$
Moreover,
$\sum_{i=1}^{r}\lambda_i
=\lVert\mat{H}\rVert_F^2
=\lVert\mat{M}\rVert_F^2$,
so each $\lambda_i$ measures the contribution of the corresponding
singular direction to the aggregate's spectral energy.

For a nonzero aggregate and an energy-retention threshold $\tau\in(0,1]$, \ours{} selects the smallest rank whose leading directions retain at least a $\tau$ fraction of this energy:
\begingroup
\thinmuskip=2.5mu
\medmuskip=3.5mu plus 1mu minus 2mu
\thickmuskip=3.5mu plus 2mu minus 2mu
\begin{equation}
	p
	=
	\min\left\{
	p'\in\{1,\ldots,r\}
	\;\middle|\;
	\frac{\sum_{i=1}^{p'}\lambda_i}
	{\sum_{i=1}^{r}\lambda_i}
	\ge\tau
	\right\}.
	\label{eq:rank-selection}
\end{equation}
\endgroup

\paragraph{Global adapter construction.}
Let $\mat{U}_p=[\mat{u}_1,\ldots,\mat{u}_p]
\in\mathbb{R}^{r\times p}$ contain the eigenvectors associated with the
$p$ largest eigenvalues of $\mat{G}$. These vectors span the dominant
left singular subspace of $\mat{H}$. Orthogonally projecting $\mat{H}$
onto this subspace therefore gives its best rank-$p$ approximation $
	\mat{H}_p
	=
	\mat{U}_p\mat{U}_p^{\top}\mat{H}.
	\label{eq:core-truncation}
$
Mapping this approximation back to the original matrix space via $\mat{Q}$ yields
\begin{equation}
	\mat{Q}\mat{H}_p
	=
	(\mat{Q}\mat{U}_p)
	(\mat{U}_p^{\top}\mat{H}).
	\label{eq:lifted-truncation}
\end{equation}
We therefore define
\begin{equation}
	\mat{L}_p
	=
	\mat{Q}\mat{U}_p
	\in\mathbb{R}^{d\times p},
	\qquad
	\mat{R}_p
	=
	\mat{U}_p^{\top}\mat{H}
	\in\mathbb{R}^{p\times D}.
	\label{eq:recovered-factors}
\end{equation}
The full basis $\mat{Q}$ need not be materialized: after zero-padding
$\mat{U}_p$ to $d\times p$, $\mat{L}_p=\mat{Q}\mat{U}_p$ can be
obtained by applying the stored Householder reflectors directly.
When $m\le n$, $\mat{M}=\overline{\Delta\mat{W}}$, so these factors
already have the required LoRA orientation. When $m>n$,
$\mat{M}=\overline{\Delta\mat{W}}^{\top}$, and the factorization is
transposed to recover the original orientation. Thus,
\begin{equation}
	(\mat{B}_g,\mat{A}_g)
	=
	\begin{cases}
		(\mat{L}_p,\mat{R}_p), & m\le n,\\
		(\mat{R}_p^{\top},\mat{L}_p^{\top}), & m>n.
	\end{cases}
	\label{eq:global-factors}
\end{equation}
In both cases, $\mat{B}_g\mat{A}_g$ is a best rank-$p$
approximation of the effective weighted aggregate
$\sum_k\alpha_k s_k\mat{B}_k\mat{A}_k$ by
Proposition~\ref{prop:core-equivalence}. The resulting pair $(\mat{B}_g,\mat{A}_g)$ is therefore a merge-ready factorization as its product already includes the client-side LoRA scaling.

\paragraph{Complexity.}
In the typical regime $p\le r\le d$, \ours{} requires
$\mathcal{O}((d+D)r^2+r^3)$ time and
$\mathcal{O}((d+D)r+r^2)$ additional workspace beyond the weighted
factor stacks; this bound includes
$\mathcal{O}((d+D)rp)$ for rank-$p$ factor recovery.
The orientation in Eq.~\eqref{eq:orientation} confines reduced Householder QR to
$\mat{X}\in\mathbb{R}^{d\times r}$ and eigendecomposition to
$\mat{G}\in\mathbb{R}^{r\times r}$, while the larger dimension $D$
appears only in matrix multiplications.
We provide derivation and FLOP comparisons in
Appendix~\ref{app:flop-comparison}.

\begin{table*}[!t]

\caption{
	Overall comparison of federated LoRA methods on
	(a) text classification and (b) commonsense reasoning.
	We report accuracy (\%), downlink payload normalized to FLoRA (\%), and aggregation latency (ms; mean $\pm$ std).
	Best accuracies among compressed-downlink methods are shown in
	\textbf{bold}, and second-best accuracies are \underline{underlined}.
	FLoRA and FedEx-LoRA are excluded from this ranking because they do not
	perform rank recompression.
}

\label{performance-comparison}

\begin{subtable}{\textwidth}
\centering
\caption{Text classification.}
\label{tab:r49-acc-comm}
\footnotesize
\setlength{\tabcolsep}{2.6pt}
\renewcommand{\arraystretch}{1.18}
\resizebox{0.98\textwidth}{!}{%
\begin{tabular}{l|cc|cc|cc|cc|cc}
\toprule
\multirow{3}{*}{\textsc{Method}}
  & \multicolumn{8}{c|}{\textbf{Accuracy} }
  & \multirow{3}{*}{\shortstack{\textsc{Downlink}\\(\% FLoRA)}}
  & \multirow{3}{*}{\shortstack{\textsc{Avg. Agg.}\\\textsc{Time} (ms)}} \\
\cmidrule(lr){2-9}
  & \multicolumn{2}{c|}{\textsc{20 Newsgroups}}
  & \multicolumn{2}{c|}{\textsc{Banking77}}
  & \multicolumn{2}{c|}{\textsc{CLINC150}}
  & \multicolumn{2}{c|}{\textsc{HWU64}}
  & & \\
  & $\operatorname{Dir}(0.1)$ & $\operatorname{Dir}(0.02)$
  & $\operatorname{Dir}(0.1)$ & $\operatorname{Dir}(0.02)$
  & $\operatorname{Dir}(0.1)$ & $\operatorname{Dir}(0.02)$
  & $\operatorname{Dir}(0.1)$ & $\operatorname{Dir}(0.02)$
  & & \\
\midrule
FLoRA
  & 64.42 & 61.65
  & 88.30 & 87.22
  & 94.79 & 95.43
  & 89.99 & 89.27
  & 100.00 & $4.3 \pm 0.2$ \\
FedEx-LoRA
  & 63.02 & 61.62
  & 86.96 & 87.00
  & 93.38 & 94.35
  & 89.78 & 88.85
  & 110.00 & $11.7 \pm 0.3$ \\
\midrule
FedIT
  & \underline{65.28} & 56.45
  & 84.33 & 79.20
  & 91.85 & 89.31
  & 88.44 & 82.46
  & 10.00 & $1.9 \pm 0.1$ \\
LoRA-A$^2$
  & 63.09 & 56.63
  & 84.18 & 81.11
  & 91.23 & 86.79
  & 88.75 & 83.49
  & 10.00 & $1.4 \pm 0.1$ \\
FlexLoRA
  & \textbf{65.50} & 59.51
  & 86.42 & 82.34
  & 92.67 & 91.09
  & \underline{89.58} & 83.90
  & 10.00 & $2722.2 \pm 7.8$ \\
FLoRG
  & 61.73 & 54.51
  & 81.80 & 77.39
  & 88.79 & 84.22
  & 87.00 & 79.98
  & 1.70 & $64.4 \pm 0.4$ \\
FLoRIST
  & 62.83 & 56.28
  & 83.57 & 78.62
  & 88.89 & 84.15
  & 88.54 & 81.53
  & 4.20 & $206.1 \pm 0.7$ \\
\midrule
\rowcolor{cyan!25} \ours{} ($\tau=0.80$)
  & 63.87 & \underline{60.95}
  & \underline{87.11} & \underline{85.66}
  & \underline{94.42} & \underline{94.81}
  & \textbf{89.99} & \underline{88.65}
  & 23.40 & $14.2 \pm 0.0$ \\
\rowcolor{cyan!25} \ours{} ($\tau=0.95$)
  & 64.29 & \textbf{61.16}
  & \textbf{87.69} & \textbf{86.82}
  & \textbf{94.57} & \textbf{95.28}
  & \textbf{89.99} & \textbf{89.06}
  & 48.70 & $14.3 \pm 0.3$ \\
\bottomrule
\end{tabular}%
}
\end{subtable}

\vspace{6pt}

\begin{subtable}{\textwidth}
\centering
\caption{Reasoning benchmarks.}
\label{tab:reasoning-benchmarks}
\footnotesize
\setlength{\tabcolsep}{2.6pt}
\renewcommand{\arraystretch}{1.18}
\resizebox{0.98\textwidth}{!}{%
\begin{tabular}{l|cccccccc|c|cc}
\toprule
\textsc{Method} & \textsc{BoolQ} & \textsc{PIQA} & \textsc{SIQA}
  & \textsc{HellaSwag} & \textsc{WinoGrande} & \textsc{ARC-E}
  & \textsc{ARC-C} & \textsc{OBQA} & \textsc{Avg.}
  & \shortstack{\textsc{Downlink}\\(\% FLoRA)}
  & \shortstack{\textsc{Avg. Agg.}\\\textsc{Time} (ms)} \\
\midrule
FLoRA
  & 88.78 & 79.49 & 78.45 & 85.34
  & 77.27 & 89.73 & 78.24 & 81.60 & 82.36 & 100.00 & $16.7 \pm 0.4$ \\
FedEx-LoRA
  & 88.75 & 78.89 & 78.15 & 83.42
  & 78.30 & 89.69 & 77.22 & 82.60 & 82.13 & 112.51 & $60.0 \pm 9.5$ \\
\midrule
FedIT
  & 87.31 & 78.67 & 76.82 & 79.64
  & 73.64 & \textbf{89.86} & \underline{77.73} & 80.40 & 80.51
  & 12.51 & $9.7 \pm 0.3$ \\
LoRA-A$^2$
  & 85.84 & 71.49 & 72.67 & 55.23
  & 64.09 & 81.27 & 63.99 & 74.00 & 71.07
  & 12.51 & $5.0 \pm 0.1$ \\
FlexLoRA
  & 87.86 & \underline{79.49} & 76.41 & 81.18
  & 75.85 & \textbf{89.86} & \textbf{78.41} & 80.40 & 81.18
  & 12.51 & $197270.4 \pm 34.3$\\
FLoRG
  & 86.97 & 77.42 & 75.13 & 74.50
  & 71.19 & 87.37 & 74.32 & 79.60 & 78.31 & 4.01 & $3403.7 \pm 9.4$\\
FLoRIST
  & 85.66 & 77.48 & 73.13 & 70.25
  & 71.03 & 88.43 & 75.51 & 76.40 & 77.24
  & 12.51 & $728.0 \pm 6.7$ \\
\midrule
\rowcolor{cyan!25} \ours{} ($\tau=0.80$)
  & \textbf{88.47} & \textbf{80.03} & \underline{77.02} & \underline{82.19}
  & \underline{76.24} & 89.18 & 76.45 & \underline{82.00} & \underline{81.45}
  & 40.01 & $86.7 \pm 1.7$ \\
\rowcolor{cyan!25} \ours{} ($\tau=0.95$)
  & \underline{88.13} & 79.33 & \textbf{77.33} & \textbf{83.28}
  & \textbf{78.45} & \underline{89.56} & 77.39 & \textbf{82.80} & \textbf{82.03}
  & 67.82 & $86.7 \pm 1.8$ \\
\bottomrule
\end{tabular}%
}
\end{subtable}
\end{table*}

\subsection{Overall Pipeline}
\label{sec:federated-pipeline}

As summarized in Figure~\ref{fig:workflow}, \ours{} retains the
merge-and-reinitialize workflow of
FLoRA~\citep{wang2024florafederatedfinetuninglarge}, but broadcasts a
recompressed global adapter instead of the full stacked adapter. The rank-$p$
global adapter is stored with $\alpha_{\mathrm{LoRA},g}=p$, giving unit global
scaling $s_g=\alpha_{\mathrm{LoRA},g}/p=1$. Therefore, merging the broadcast
adapter applies exactly the merge-ready patch $\mat{B}_g\mat{A}_g$, without
an additional scaling factor. At the beginning of round $t$, each client merges the global adapter broadcast
in the preceding round into its cached global weights:
\[
\mat{W}_t
=
\mat{W}_{t-1}
+
\mat{B}_g^{(t-1)}\mat{A}_g^{(t-1)},
\!\qquad
\mat{B}_g^{(0)}\mat{A}_g^{(0)}=\mat{0}.
\]
Each client then initializes a fresh rank-$r_k$ residual adapter
$(\mat{B}_k^{(t)},\mat{A}_k^{(t)})$, fine-tunes it on top of frozen $\mat{W}_t$, and uploads the resulting factors. Reinitialization prevents the previously merged global adapter from being counted again
in the next aggregate. The server forms the exact weighted factor
stacks according to Eq.~\eqref{eq:weighted-stacking} and applies
\ours{} to produce the next compact global adapter
$(\mat{B}_g^{(t)},\mat{A}_g^{(t)})$ for broadcast.

The weighted factor stacks are algebraically equivalent to explicitly
aggregating the full-size local updates. Consequently, \ours{}
recompresses the same aggregate without requiring a common local
factor rank, naturally accommodating heterogeneous client ranks
$\{r_k\}$. The complete round-by-round procedure is provided in
Appendix~\ref{app:algorithm}.

\section{Experiments}
\label{sec:experiments}

In this section, we evaluate the performance of our \ours{} against existing federated LoRA methods across various heterogeneity and homogeneous rank settings and datasets. We assess performance based on accuracy, downlink overhead, and server-side recompression efficiency.

\begin{table*}[!t]
\centering

\caption{
	Heterogeneous-rank results on commonsense reasoning benchmarks.
	We report accuracy (\%), downlink payload normalized to FLoRA (\%), and aggregation latency (ms; mean $\pm$ std). Best accuracies among compressed-downlink methods are shown in
	\textbf{bold}, and second-best accuracies are \underline{underlined}.
}

\label{tab:hetero-rank}
\footnotesize
\setlength{\tabcolsep}{2.6pt}
\renewcommand{\arraystretch}{1.18}
\resizebox{0.98\textwidth}{!}{%
\begin{tabular}{l|cccccccc|c|cc}
\toprule
\textsc{Method} & \textsc{BoolQ} & \textsc{PIQA} & \textsc{SIQA}
  & \textsc{HellaSwag} & \textsc{WinoGrande} & \textsc{ARC-E}
  & \textsc{ARC-C} & \textsc{OBQA} & \textsc{Avg.}
  & \shortstack{\textsc{Downlink}\\(\% FLoRA)}
  & \shortstack{\textsc{Avg. Agg.}\\\textsc{Time} (ms)} \\
\midrule
FLoRA
  & 89.30 & 76.77 & 77.64 & 81.05
  & 74.27 & 88.01 & 72.18 & 78.60 & 79.73
  & 100.0 & $18.9 \pm 0.7$ \\
\midrule
FlexLoRA
  & \underline{88.62} & \underline{77.04} & \underline{75.69} & \underline{79.99}
  & 72.69 & \textbf{88.13} & 69.71 & 75.60 & 78.44
  & 27.6 & $334587.6 \pm 1576.2$ \\
FLoRIST
  & 88.47 & 76.99 & 74.92 & 79.71
  & \underline{73.01} & \underline{88.09} & \underline{71.76} & \underline{77.60} & \underline{78.82}
  & 23.0 & $4364.6 \pm 10.5$ \\
\midrule
\rowcolor{cyan!25} \ours{} ($\tau=0.80$)
  & \underline{88.62} & 76.22 & 75.08 & 78.16
  & 72.14 & 87.67 & \textbf{71.93} & 76.60 & 78.30
  & {16.2} & $369.4 \pm 5.3$ \\
\rowcolor{cyan!25} \ours{} ($\tau=0.95$)
  & \textbf{88.93} & \textbf{77.31} & \textbf{76.10} & \textbf{81.01}
  & \textbf{74.19} & 87.63 & 71.33 & \textbf{78.20} & \textbf{79.34}
  & 45.5 & $367.9 \pm 4.0$ \\
\bottomrule
\end{tabular}%
}
\end{table*}

\subsection{Experimental Setup}
\paragraph{Tasks, federated settings, and metrics.}
We evaluate \ours{} with RoBERTa-base~\citep{liu2019roberta} on four intent
and topic classification datasets: BANKING77~\citep{casanueva2020efficient},
20 Newsgroups~\citep{lang1995newsweeder},
CLINC150~\citep{larson2019evaluation}, and HWU64~\citep{liu2019benchmarking}.
Each dataset is partitioned across 10 clients using Dirichlet label splits with
$\alpha\in\{0.1,0.02\}$, representing moderate and extreme heterogeneity.
Training runs for 30 communication rounds with 200 local steps per round. For this model, we use a single linear classification head for all methods. The
classifier is jointly trained with the local LoRA adapters and aggregated using
FedAvg during the first five communication rounds, after which it is frozen.
This ensures that subsequent performance differences primarily reflect the
LoRA adaptation and aggregation mechanisms.

We fine-tune LLaMA-3.2-3B~\citep{llama32} on the
Commonsense-170K reasoning benchmark following FedEx-LoRA~\citep{singhal-etal-2025-fedex}. This setting uses eight
clients with a pathological cross-task partition, where each client holds one
reasoning task. The main evaluation uses a uniform local LoRA rank of 8. We
additionally consider heterogeneous local ranks
$(r_1,\ldots,r_8)=(4,16,32,64,4,16,32,64)$ under the same reasoning setup,
with all other training settings unchanged. Dataset details and complete
training hyperparameters are provided in Appendices~\ref{app:datasets}
and~\ref{app:hyperparameters}.

We evaluate \ours{} at energy thresholds $\tau=0.80$ and $0.95$ to
characterize the accuracy--communication trade-off. We measure per-client
downlink by the number of LoRA parameters broadcast per round, normalized to
the uncompressed FLoRA stack ($100\%$). We evaluate server-side
aggregation efficiency using wall-clock latency and peak memory, complemented
by the theoretical FLOP analysis in Appendix~\ref{app:flop-comparison}.
Latency is reported as mean $\pm$ standard deviation and includes
recompression where applicable. Federated training is simulated on a cluster of NVIDIA GH200 GPUs, whereas server-side efficiency is measured on a single GPU from this cluster under identical software conditions for all methods.  Unless otherwise noted, weight matrices with compatible shapes are processed using batched GPU kernels.

\paragraph{Baselines.}
We compare \ours{} against factor-space aggregation methods
FedIT~\citep{10447454} and LoRA-$A^2$~\citep{koo2025towards}, exact and uncompressed aggregation
methods FLoRA~\citep{wang2024florafederatedfinetuninglarge} and
FedEx-LoRA~\citep{singhal-etal-2025-fedex}, and recompression methods
FlexLoRA~\citep{bai2024federated}, FLoRG~\citep{meng2026florg}, and
FLoRIST~\citep{ramesh2026florist}. Detailed baseline descriptions are provided
in Appendix~\ref{app:baselines}. The heterogeneous-rank evaluation includes
only methods that natively support varying client ranks, namely FLoRA,
FlexLoRA, FLoRIST, and \ours{}.

\begin{figure*}[!t]
\centering
\begingroup
\definecolor{fraqblue}{HTML}{6A3D9A}%
\definecolor{floristgreen}{HTML}{1B9E77}%
\definecolor{flexorange}{HTML}{E66101}%
\small
\textcolor{fraqblue}{$\bullet$}\,\ours{}\qquad
\textcolor{floristgreen}{$\blacksquare$}\,FLoRIST\qquad
\textcolor{flexorange}{$\blacktriangle$}\,FlexLoRA\qquad
\tikz[baseline=-0.5ex]{\draw[line width=1.2pt] (0,0)--(0.38,0); \fill (0.19,0) circle (1.35pt);}\,Batched\qquad
\tikz[baseline=-0.5ex]{\draw[line width=1pt,dash pattern=on 3pt off 2pt] (0,0)--(0.62,0); \draw[fill=white,line width=0.7pt] (0.31,0) circle (1.35pt);}\,Unbatched
\par\vspace{1pt}
\endgroup
\begin{subfigure}[t]{0.493\textwidth}
  \centering
  \includegraphics[width=\linewidth]{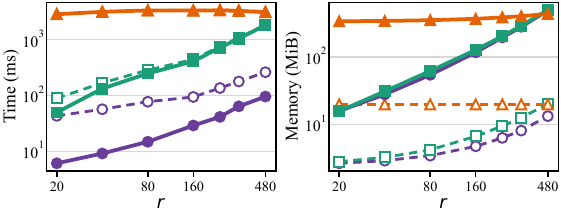}
  \caption{Scalability versus total rank $r$}
\end{subfigure}\hfill
\begin{subfigure}[t]{0.493\textwidth}
  \centering
  \includegraphics[width=\linewidth]{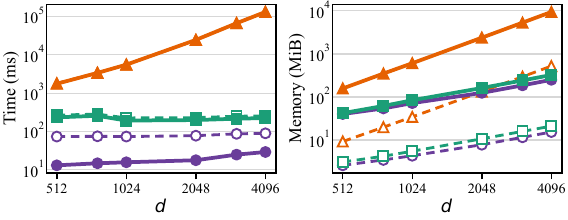}
  \caption{Scalability versus hidden size $d$}
\end{subfigure}
\caption{Server-side recompression runtime and peak memory as
	(a) the total stacked rank $r$ varies at $d=768$ and
	(b) the hidden size $d$ varies at $r=80$. Solid and dashed curves denote
	batched and unbatched execution, respectively. \ours{} consistently achieves
	the lowest runtime while maintaining low peak memory across both sweeps.}
\label{fig:agg-scalability}
\end{figure*}

\subsection{Main Results}
\label{sec:main-results}

\paragraph{Accuracy and communication.}
Table~\ref{performance-comparison} reports accuracy, normalized downlink, and
aggregation latency. Boldface and underlining mark the best and second-best
compressed results; unrecompressed FLoRA and FedEx-LoRA are excluded.
\ours{} ($\tau{=}0.95$) achieves the best compressed result on three of four
classification datasets under moderate heterogeneity and all four under
extreme heterogeneity, remaining within $0.61$ percentage points of FLoRA
across all eight settings. \ours{} ($\tau{=}0.80$) ranks second on all four
datasets under extreme heterogeneity. On commonsense reasoning, the two
variants rank first and second among compressed methods at $82.03\%$ and
$81.45\%$, respectively, with $\tau{=}0.95$ only $0.33$ points below FLoRA.

At $\tau{=}0.80$, \ours{} uses $23.4\%$ and $40.0\%$ of the FLoRA downlink
for classification and reasoning, respectively; $\tau{=}0.95$ increases these
values to $48.7\%$ and $67.8\%$ in exchange for higher accuracy. FLoRG and
FLoRIST use less downlink but incur substantially larger accuracy losses,
particularly under extreme label heterogeneity.

\paragraph{Heterogeneous client ranks.}
Table~\ref{tab:hetero-rank} evaluates methods that natively support
heterogeneous local ranks. \ours{} ($\tau{=}0.95$) achieves the highest average
accuracy among compressed methods at $79.34\%$, only $0.39$ percentage points
below FLoRA while using $45.5\%$ of its downlink. \ours{} ($\tau{=}0.80$)
reduces the downlink further to $16.2\%$, the lowest among all compared methods,
while retaining an average accuracy of $78.30\%$.

\begin{figure}[!t]
\centering
\begingroup
\definecolor{fraqblue}{HTML}{6A3D9A}%
\definecolor{floristgreen}{HTML}{1B9E77}%
\definecolor{flexorange}{HTML}{E66101}%
\scriptsize
\textcolor{fraqblue}{$\bullet$}\,\ours{}\qquad
\textcolor{floristgreen}{$\blacksquare$}\,FLoRIST\quad
\textcolor{flexorange}{$\blacktriangle$}\,FlexLoRA\quad
\tikz[baseline=-0.5ex]{\draw[line width=1pt,dash pattern=on 4pt off 2.5pt,color=gray!80!black] (0,0)--(0.55,0);}\,Optimal
\par\vspace{2pt}
\endgroup
\includegraphics[width=0.995\linewidth]{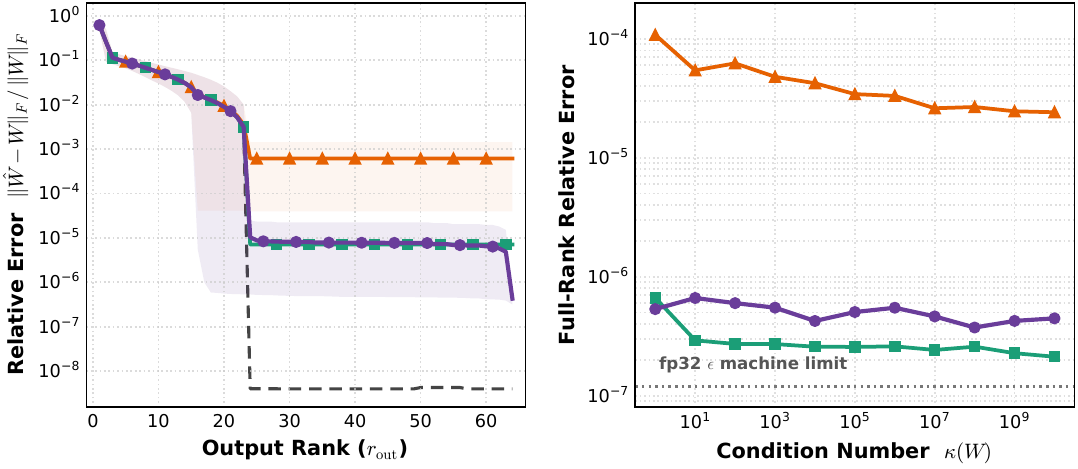}
\caption{Numerical fidelity and stability in FP32 against FP64 references.
	The left panel reports mean error over LLaMA-3.2-3B LoRA updates
	(shading: 10th--90th percentiles), while the right tests full-rank
	reconstruction under controlled condition numbers. \ours{} tracks the
	Eckart--Young optimum and remains near $10^{-7}$ through
	$\kappa(\mat{W})=10^{10}$.}
\label{fig:numerical-fidelity}
\end{figure}

\paragraph{Server-side efficiency.}
\ours{} requires 14.2~ms and 86.7~ms per aggregation for
RoBERTa-base and LLaMA-3.2-3B, respectively, making it
$8$--$15\times$ faster than FLoRIST and
$190$--$2{,}300\times$ faster than FlexLoRA. In the heterogeneous-rank
experiments, both variants require approximately 0.37~s, around
$12\times$ faster than FLoRIST and more than $900\times$ faster than
FlexLoRA.

Figure~\ref{fig:agg-scalability} further examines scalability with total
stacked rank and hidden size. Across both sweeps, \ours{} remains the fastest
method while maintaining low peak memory under batched and unbatched
execution. Unlike FlexLoRA, whose dense reconstruction causes runtime and
memory to grow rapidly with hidden size, \ours{} operates on stacked factors;
its smaller-side QR and Gram recovery also avoid the repeated rectangular SVDs
used by FLoRIST. These trends are consistent with the FLOP analysis in
Appendix~\ref{app:flop-comparison}.

\paragraph{Numerical fidelity and stability.}
Figure~\ref{fig:numerical-fidelity} evaluates all methods in FP32 with TF32
disabled against FP64 references. The left panel uses 196 LoRA modules from
one LLaMA-3.2-3B federated round with eight rank-8 clients. \ours{} closely
tracks the Eckart--Young optimum over the signal-bearing rank range and
matches the fidelity of FLoRIST. The sharp decrease in optimal error near rank
24, despite a formal stacked rank of 64, indicates substantial spectral
redundancy among client updates. The right panel isolates finite-precision
stability using rank-64 synthetic aggregates with condition numbers from $1$
to $10^{10}$. Because the output rank is also 64, the measured error contains
no intentional truncation error.

\newpage
\begin{table}[!h]
	\centering
	\scriptsize
	\setlength{\tabcolsep}{6.pt}
	\caption{Ablation study on real LLaMA-3.2-3B LoRA updates.
	Relative error is measured against the coordinate-SVD rank-8 reference.}
	\label{tab:algorithmic-ablation}
	\begin{tabular}{@{}lrrr@{}}
		\toprule
		Variant &
		Time (ms) $\downarrow$ &
		Memory (MiB) $\downarrow$ &
		Rel. error $\downarrow$ \\
		\midrule
		Adaptive + Gram
		& \textbf{14.21} & \textbf{90.13} & $2.56\times10^{-4}$ \\
		Fixed + Gram
		& 25.98 & \textbf{90.13} & $8.51\times10^{-4}$ \\
		\shortstack[l]{Adaptive +\\coordinate SVD}
		& 62.26 & 199.02 & Reference \\
		\bottomrule
\end{tabular}
\end{table}

\noindent \ours{} maintains a relative error of
approximately $10^{-7}$ throughout the sweep, close to FLoRIST and roughly
two orders of magnitude below FlexLoRA, demonstrating stable FP32
reconstruction even for severely ill-conditioned aggregates.

\paragraph{Ablation study.}
Table~\ref{tab:algorithmic-ablation} isolates the two main algorithmic choices
on 56 LLaMA-3.2-3B gate and up projections of shape
$8192\times3072$. Direction-adaptive orientation reduces latency by
$1.83\times$ without increasing peak memory, while Gram-based recovery is
$4.38\times$ faster and uses $54.7\%$ less memory than coordinate SVD.
Both Gram variants remain within $0.1\%$ of the rank-8 reference.

\FloatBarrier

\section{Conclusion}
\label{sec:conclusion}

We presented \ours{}, a coordinate-space recompression framework for federated LoRA aggregation. Starting from weighted stacks that exactly represent the aggregated update, \ours{} uses a reduced QR factorization and a compact Gram eigenproblem to recover its spectrum without materializing the dense update or performing a full-matrix SVD. Experiments across diverse federated fine-tuning settings demonstrate that \ours{} achieves a favorable trade-off among model accuracy, downlink communication, and server-side aggregation latency, making federated fine-tuning more practical in heterogeneous and resource-constrained environments.

\section*{Limitations}
We note two limitations of \ours{}. First, the cost of coordinate-space aggregation scales with the total stacked rank of the participating clients in each round, that is, the sum of their local LoRA ranks. When a large cohort of clients participates per round, this total rank grows accordingly, and the reduced QR factorization, the Gram-matrix construction, and the eigendecomposition all become more expensive. Consequently, the server-side recompression efficiency degrades as the number of participating clients increases, and the advantage of \ours{} over full-weight SVD shrinks in the regime of very large per-round participation. A detailed cost breakdown is given in the appendix.

Second, like FLoRA, \ours{} follows a merge-and-reinitialize scheme: in every round the consolidated update is baked into the global model and clients reinitialize fresh adapters, so local adapters are not warm-started from the previous round as in PEFT-style methods that reuse the downloaded adapter. This can slow per-adapter convergence relative to methods that keep refining a persistent adapter. Importantly, however, the coordinate-space recompression at the heart of \ours{} is not tied to the merge-and-reinitialize paradigm. The same QR-and-Gram construction could potentially replace SVD-based recompression in non-merge-and-reinitialize methods such as \texttt{FLoRIST} and FLoRG, but validating the resulting accuracy and runtime is left to future work.

\bibliography{references}

\appendix
\section{Theoretical Guarantees and Complexity Analysis}
\label{app:theory}

\subsection{Proof of Proposition~\ref{prop:core-equivalence}}
\label{app:proof}

\begin{proof}
Because $\operatorname{col}(\mat{X})\subseteq\operatorname{col}(\mat{Q})$,
the matrix $\mat{Q}\mat{Q}^{\top}$ acts as the identity on $\mat{X}$:
\begin{equation}
    \mat{Q}\mat{Q}^{\top}\mat{X}=\mat{X}.
\end{equation}
Therefore,
\begin{equation}
    \mat{Q}\mat{H}
    =\mat{Q}\mat{Q}^{\top}\mat{X}\mat{Y}
    =\mat{X}\mat{Y}
    =\mat{M}.
\end{equation}

Moreover, for any
$\widehat{\mat{H}}\in\mathbb{R}^{r\times D}$,
\begin{align*}
\left\|\mat{M}-\mat{Q}\widehat{\mat{H}}\right\|_F
&=
\left\|\mat{Q}
\left(\mat{H}-\widehat{\mat{H}}\right)\right\|_F \\
&=
\left\|\mat{H}-\widehat{\mat{H}}\right\|_F,
\end{align*}
where the last equality follows from
$\mat{Q}^{\top}\mat{Q}=\mat{I}_r$. This proves
Eq.~\eqref{eq:coordinate-isometry}.

Let $\mat{H}=\mat{U}\mat{\Sigma}\mat{V}^{\top}$ be a compact SVD. Since both
$\mat{Q}$ and $\mat{U}$ have orthonormal columns,
\begin{equation}
    (\mat{Q}\mat{U})^{\top}(\mat{Q}\mat{U})
    =\mat{U}^{\top}\mat{Q}^{\top}\mat{Q}\mat{U}
    =\mat{I}.
\end{equation}
Therefore
$\mat{M}=(\mat{Q}\mat{U})\mat{\Sigma}\mat{V}^{\top}$ is a compact SVD of
$\mat{M}$. Thus, $\mat{H}$ and $\mat{M}$ share the same nonzero singular
values.

Truncating this SVD to its leading $p$ components and applying the
Eckart--Young--Mirsky theorem~\cite{golub1987generalization} proves
Eq.~\eqref{eq:core-equivalence}.
\end{proof}

\subsection{General-Rank Formulation and Asymptotic Complexity}
\label{app:general-complexity}

For completeness, we first analyze the general case in which the total stacked
rank may exceed the smaller weight dimension. Let
\begin{equation}
    \begin{aligned}
        d&=\min(m,n), & D&=\max(m,n),\\
        r&=\sum_{k=1}^{K}r_k, & s&=\min(d,r).
    \end{aligned}
\end{equation}
After orientation, $\mat{X}\in\mathbb{R}^{d\times r}$ and
$\mat{Y}\in\mathbb{R}^{r\times D}$. A reduced Householder QR
factorization obtains $\mat{R}\in\mathbb{R}^{s\times r}$ and an implicit
representation of $\mat{Q}\in\mathbb{R}^{d\times s}$, followed by
$\mat{H}=\mat{R}\mat{Y}\in\mathbb{R}^{s\times D}$ and
$\mat{G}=\mat{H}\mat{H}^{\top}\in\mathbb{R}^{s\times s}$. The per-matrix
costs are
\begin{align}
    \text{reduced QR:}\quad
    &\mathcal{O}(drs),\\
    \text{coordinate construction:}\quad
    &\mathcal{O}(srD),\\
    \text{Gram construction:}\quad
    &\mathcal{O}(s^2D),\\
    \text{eigendecomposition:}\quad
    &\mathcal{O}(s^3),\\
    \text{patch recovery:}\quad
    &\mathcal{O}\!\left(ps(d+D)\right).
\end{align}
Thus, the formulation remains valid when $r>d$, in which case $s=d$. In the
typical regime $r\le d$, we have $s=r$, and the total time reduces to
$\mathcal{O}(dr^2+r^2D+r^3)$, typically dominated by the coordinate and Gram
constructions.

The direction-adaptive orientation ensures that the Householder representation
of $\mat{Q}$ is computed only from the factor associated with the smaller
dimension $d$. If the factorization were
always anchored on $\mat{B}$, QR would instead cost $\mathcal{O}(Dr^2)$ when
$m>n$, compared with \ours{}'s $\mathcal{O}(dr^2)$, a factor-of-$D/d$ reduction.
The additional workspace for the Householder representation, $\mat{R}$,
$\mat{H}$, and $\mat{G}$ is
$\mathcal{O}\!\left(s(d+D+r)+s^2\right)$, which reduces to
$\mathcal{O}((d+D)r+r^2)$ in the typical regime. These bounds exclude the
resident stacked factors representing the exact aggregate. Direct dense
recompression additionally materializes the $d\times D$ update, costing
$\mathcal{O}(dDr)$ computation and $\mathcal{O}(dD)$ workspace before its
matrix decomposition.

For $L$ adapted weight matrices, the total server cost is the sum of the
per-matrix costs,
\begin{equation}
    F_{\mathrm{total}}
    =\sum_{\ell=1}^{L}
    F\!\left(d_\ell,D_\ell,r_\ell,p_\ell\right),
    \label{eq:multilayer-cost}
\end{equation}
which accounts for the different shapes and selected ranks of transformer
projections.

\subsection{Detailed FLOP Comparison}
\label{app:flop-comparison}

We next provide a leading-order arithmetic comparison among \ours{}, FLoRIST,
and FlexLoRA. Let $p\le r$ denote the retained global rank. We focus on the
typical LoRA regime
\begin{equation}
    p\le r\ll d\le D.
\end{equation}

For a coarse leading-order comparison, we use the following standard
arithmetic estimates:
\begin{align}
    F_{\mathrm{SVD}}(a,r)
    &\approx 4ar^2+\mathcal{O}(r^3),
    \label{eq:flop-thin-svd}\\
    F_{\mathrm{QR}}(a,r)
    &\approx 2ar^2-\frac{2}{3}r^3,
    \label{eq:flop-qr}\\
    F_{\mathrm{GEMM}}(a,b,c)
    &\approx 2abc.
    \label{eq:flop-gemm}
\end{align}
Here, Eq.~\eqref{eq:flop-thin-svd} approximates a thin SVD of an $a\times r$
matrix, Eq.~\eqref{eq:flop-qr} approximates reduced Householder QR, and
Eq.~\eqref{eq:flop-gemm} approximates multiplication of an $a\times b$ matrix
by a $b\times c$ matrix. These estimates are intended only to expose the
leading dimension-dependent terms; exact constants depend on the numerical
routine, whether singular vectors are explicitly formed, and the underlying
hardware implementation.

\paragraph{FlexLoRA.}
FlexLoRA explicitly materializes the dense aggregate
$\overline{\Delta\mat{W}}\in\mathbb{R}^{d\times D}$ before decomposition.
Even under a favorable implementation that forms the aggregate through one
multiplication of the weighted stacked factors, dense materialization costs
approximately $2dDr$ FLOPs. An economy SVD of the resulting $d\times D$
matrix, with $d\le D$, has a leading cost of approximately
$4Dd^2+\mathcal{O}(d^3)$. Recovering the top-$p$ factors requires only
lower-order scaling and slicing operations. We therefore obtain
\begin{equation}
    F_{\mathrm{FlexLoRA}}
    \approx 2dDr+4Dd^2
    +\mathcal{O}\!\left(d^3+(d+D)p\right).
    \label{eq:flops-flexlora}
\end{equation}
This estimate is conservative in favor of FlexLoRA because it omits any
additional accumulation cost incurred when client updates are materialized
separately.

\paragraph{FLoRIST.}
FLoRIST first applies thin SVDs to the two stacked factors, whose dimensions
are $d\times r$ and $D\times r$ after orienting them consistently. Their
leading cost is
\begin{equation}
    4dr^2+4Dr^2=4(d+D)r^2.
\end{equation}
FLoRIST then forms and decomposes an $r\times r$ intermediate matrix,
contributing an additional $\mathcal{O}(r^3)$ term. Recovering rank-$p$ global
factors requires multiplying the rectangular singular-vector bases by the
retained intermediate factors, contributing approximately $2(d+D)rp$ FLOPs.
Its total leading cost is therefore
\begin{equation}
    F_{\mathrm{FLoRIST}}
    \approx 4(d+D)r^2+2(d+D)rp+\mathcal{O}(r^3).
    \label{eq:flops-florist}
\end{equation}

\paragraph{\ours{}.}
\ours{} performs reduced QR only on the factor anchored along the smaller
weight dimension. The QR factorization of the $d\times r$ matrix costs
approximately $2dr^2-\frac{2}{3}r^3$ FLOPs. Constructing the coordinate matrix
\begin{equation}
    \mat{H}=\mat{R}\mat{Y},
    \qquad
    \mat{R}\in\mathbb{R}^{r\times r},
    \quad
    \mat{Y}\in\mathbb{R}^{r\times D},
\end{equation}
costs approximately $2Dr^2$ FLOPs. Forming the symmetric Gram matrix
$\mat{G}=\mat{H}\mat{H}^{\top}$ with a symmetry-aware rank-$k$ update costs
approximately $Dr^2$ FLOPs; a generic matrix multiplication that explicitly
computes both triangular halves would instead cost approximately $2Dr^2$.

The symmetric eigendecomposition of
$\mat{G}\in\mathbb{R}^{r\times r}$ contributes $\mathcal{O}(r^3)$. Finally,
recovering the rank-$p$ merge factors requires applying the stored
Householder reflectors to $\mat{U}_p$ and computing
$\mat{U}_p^{\top}\mat{H}$, with a combined cost of
$\mathcal{O}((d+D)rp)$. Thus, when the Gram matrix is formed using symmetry,
\ours{} has the leading cost
\begin{equation}
    F_{\mathrm{FRAQ}}
    \approx (2d+3D)r^2
    +\mathcal{O}\!\left((d+D)rp+r^3\right).
    \label{eq:flops-fraq}
\end{equation}
If the full Gram matrix is formed with generic GEMM, the coefficient $3D$ in
Eq.~\eqref{eq:flops-fraq} becomes $4D$.

\begin{table*}[t]
\centering
\scriptsize
\setlength{\tabcolsep}{3pt}
\renewcommand{\arraystretch}{1.25}
\caption{Leading-order server-side FLOP comparison for one adapted weight
matrix. We assume $p\le r\ll d\le D$, where $d=\min(m,n)$,
$D=\max(m,n)$, and $r=\sum_k r_k$. The \ours{} Gram matrix is assumed to be
formed with a symmetry-aware rank-$k$ update. The displayed FLOPs are
approximate leading-order quantities rather than exact operation counts.}
\label{tab:flop-comparison}
\resizebox{\textwidth}{!}{%
\begin{tabular}{lccc}
\toprule
\textbf{Method} & \textbf{Matrix factorizations} &
\textbf{Leading-order FLOPs} & \textbf{Dominant scaling} \\
\midrule
FlexLoRA & $\operatorname{SVD}(d\times D)$ &
$2dDr+4Dd^2+\mathcal{O}\!\left(d^3+(d+D)p\right)$ &
$\Theta(Dd^2+dDr)$ \\[2mm]
FLoRIST & $\operatorname{SVD}(d\times r)+\operatorname{SVD}(D\times r)
+\operatorname{SVD}(r\times r)$ &
$4(d+D)r^2+2(d+D)rp+\mathcal{O}(r^3)$ &
$\Theta(Dr^2+r^3)$ \\[2mm]
\ours{} & $\operatorname{QR}(d\times r)+\operatorname{EVD}(r\times r)$ &
$(2d+3D)r^2+\mathcal{O}\!\left((d+D)rp+r^3\right)$ &
$\Theta(Dr^2+r^3)$ \\
\bottomrule
\end{tabular}%
}
\end{table*}

\paragraph{Comparison with FLoRIST.}
Equations~\eqref{eq:flops-florist} and~\eqref{eq:flops-fraq} show that \ours{}
and FLoRIST share the same coarse asymptotic order,
$\Theta(Dr^2+r^3)$, but \ours{} performs substantially less
decomposition-heavy computation. FLoRIST applies thin SVDs to both rectangular
factors, whereas \ours{} applies reduced QR only to the factor associated with
the smaller weight dimension. The rectangular decomposition FLOPs are
\begin{align}
    F_{\mathrm{dec}}^{\mathrm{FLoRIST}}
    &\approx 4(d+D)r^2,\\
    F_{\mathrm{dec}}^{\mathrm{FRAQ}}
    &\approx 2dr^2.
\end{align}
Consequently, \ours{} reduces the dimension-dependent decomposition workload
by approximately
\begin{equation}
    \frac{F_{\mathrm{dec}}^{\mathrm{FLoRIST}}}
         {F_{\mathrm{dec}}^{\mathrm{FRAQ}}}
    \approx 2\left(1+\frac{D}{d}\right).
    \label{eq:decomposition-reduction}
\end{equation}
For square matrices, Eq.~\eqref{eq:decomposition-reduction} corresponds to an
approximately $4\times$ reduction in rectangular decomposition FLOPs. The
reduction increases for highly rectangular projection matrices.

Considering all leading $r^2$ terms and neglecting the common rank-$p$
reconstruction term, the arithmetic ratio is approximately
\begin{equation}
    \frac{F_{\mathrm{FLoRIST}}}{F_{\mathrm{FRAQ}}}
    \approx \frac{4(d+D)}{2d+3D}.
    \label{eq:florist-fraq-ratio}
\end{equation}
For $D\ge d$, this ratio ranges from approximately $1.33\times$ to
$1.60\times$. The larger empirical wall-clock gains arise because \ours{}
handles the dependence on $D$ through GEMM/SYRK kernels rather than additional
SVDs, which generally have higher synchronization, workspace, and
kernel-launch overhead.

\paragraph{Comparison with FlexLoRA.}
The advantage over FlexLoRA is also asymptotic. FlexLoRA's dense SVD contributes
the dominant term $\Theta(Dd^2)$, whereas \ours{} operates on the stacked rank
and requires $\Theta(Dr^2)$ work in the typical regime. Because $r\ll d$,
\ours{} replaces quadratic dependence on the smaller model dimension $d$ with
quadratic dependence on the stacked rank $r$. Let
$\gamma=D/d\ge 1$. Neglecting the lower-order dense-materialization,
rank-$p$ reconstruction, and cubic terms, the approximate FLOP ratio is
\begin{equation}
    \frac{F_{\mathrm{FlexLoRA}}}{F_{\mathrm{FRAQ}}}
    \approx
    \frac{4\gamma}{2+3\gamma}\left(\frac{d}{r}\right)^2.
    \label{eq:flexlora-fraq-ratio}
\end{equation}
Thus, the arithmetic advantage of \ours{} over FlexLoRA grows quadratically
with the dimension-to-rank ratio $d/r$, consistent with the scalability
results in Figure~\ref{fig:agg-scalability}.

\FloatBarrier

\section{Complete \ours{} Procedure}
\label{app:algorithm}

Algorithm~\ref{alg:fraq} provides the complete round-by-round procedure for
\ours{}, including client-side merge-and-reinitialization and server-side
recompression. Reinitializing the residual adapters after merging the previous
global adapter prevents their old factorization from being counted again against the
updated global base. Although our storage convention gives unit global scaling,
Algorithm~\ref{alg:fraq} writes $s_g$ explicitly in each merge operation.

\begin{algorithm}[!ht]
\small
\caption{End-to-end federated training with \ours{}.}
\label{alg:fraq}
\DontPrintSemicolon
\KwIn{Pretrained weights $\mat{W}_0$, rounds $T$, clients $\mathcal{C}$ with ranks $\{r_k\}$, LoRA scales $\{s_k\}$, dataset sizes $\{n_k\}$, energy threshold $\tau\in(0,1]$}
\KwOut{Merged global model $\mat{W}_{T+1}$}
$\mat{W}_1 \gets \mat{W}_0$ \;
$s_g^{(0)}\mat{B}_g^{(0)}\mat{A}_g^{(0)}\gets\mat{0}$ \;
\For{$t=1$ \KwTo $T$}{
    \textcolor{blue}{\textbf{Server:}} broadcast the round-$(t{-}1)$ global adapter $(\mat{B}_g^{(t-1)},\mat{A}_g^{(t-1)},s_g^{(t-1)})$ to all clients (empty when $t=1$) \;
    \ForEach(\textbf{in parallel}){$k \in \mathcal{C}$}{
        Merge the received global adapter: $\mat{W}_t \gets \mat{W}_{t-1}+s_g^{(t-1)}\mat{B}_g^{(t-1)}\mat{A}_g^{(t-1)}$ \;
        Initialize a fresh residual adapter $(\mat{B}_k,\mat{A}_k)$ of rank $r_k$ \;
        $(\mat{B}_k,\mat{A}_k) \gets \texttt{LocalUpdate}(\mat{W}_t,\mat{B}_k,\mat{A}_k)$ \;
        \texttt{Upload}$(\mat{B}_k,\mat{A}_k)$ to server \;
    }
    \textcolor{blue}{\textbf{Server:}} form scaled and weighted stacks \;
    $\mat{B}\gets[\sqrt{\alpha_1}\mat{B}_1,\ldots,\sqrt{\alpha_K}\mat{B}_K]$ \;
    $\mat{A}\gets[\sqrt{\alpha_1}s_1\mat{A}_1;\ldots;\sqrt{\alpha_K}s_K\mat{A}_K]$ \;
    Orient: $(\mat{X},\mat{Y})\gets(\mat{B},\mat{A})$ if $m\leq n$; else $(\mat{X},\mat{Y})\gets(\mat{A}^{\top},\mat{B}^{\top})$ \;
    $s\gets\min(d,r)$ \;
    Compute Householder QR of $\mat{X}$, storing reflectors and $\mat{R}$ \;
    Form $\mat{H}\gets\mat{R}\mat{Y}$ and $\mat{G}\gets\mat{H}\mat{H}^{\top}$ \;
    Eigendecompose $\mat{G}=\mat{U}\mat{\Lambda}\mat{U}^{\top}$ and order the eigenpairs by decreasing $\lambda_i$ \;
    Select the smallest $p\in\{1,\ldots,s\}$ such that $\sum_{i=1}^{p}\lambda_i/\sum_{i=1}^{s}\lambda_i\geq\tau$; retain $\mat{U}_p$ \;
    $\mat{L}_p\gets\operatorname{ApplyQ}(\text{reflectors},\mat{U}_p)$ and $\mat{R}_p\gets\mat{U}_p^{\top}\mat{H}$ \;
    $(\mat{B}_g^{(t)},\mat{A}_g^{(t)})\gets(\mat{L}_p,\mat{R}_p)$ if $m\leq n$; else $(\mat{B}_g^{(t)},\mat{A}_g^{(t)})\gets(\mat{R}_p^{\top},\mat{L}_p^{\top})$ \;
    Set $\alpha_{\mathrm{LoRA},g}^{(t)}\gets p$ and $s_g^{(t)}\gets\alpha_{\mathrm{LoRA},g}^{(t)}/p=1$; store the rank-$p$ global adapter \;
    $\mat{W}_{t+1}\gets\mat{W}_t+s_g^{(t)}\mat{B}_g^{(t)}\mat{A}_g^{(t)}$ \;
}
\KwRet{$\mat{W}_{T+1}$}
\end{algorithm}
\FloatBarrier

\section{Experimental Details}
\label{app:experiments}

\subsection{Datasets and Data Partitions}
\label{app:datasets}

This appendix expands on the datasets used in our two evaluation settings: the
federated intent/topic classification benchmark (Table~\ref{tab:r49-acc-comm})
and the commonsense reasoning benchmark used in the additional evaluation
(Table~\ref{tab:reasoning-benchmarks}).

\paragraph{Intent and topic classification.}
\label{appendix:datasets-classification}
For the federated RoBERTa-base experiments we use four English text
classification datasets:
\begin{itemize}
  \item \textbf{BANKING77}~\citep{casanueva2020efficient}: a fine-grained
    single-domain intent detection dataset of online-banking customer queries
    spanning 77 intents.
  \item \textbf{CLINC150}~\citep{larson2019evaluation}: a multi-domain intent
    classification benchmark covering 150 intents across 10 domains, designed to
    also include out-of-scope queries.
  \item \textbf{HWU64}~\citep{liu2019benchmarking}: a multi-domain
    home-assistant intent dataset with 64 intents across 21 domains.
  \item \textbf{20 Newsgroups}~\citep{lang1995newsweeder}: a classic topic
    classification benchmark of newsgroup posts grouped into 20 categories.
\end{itemize}
To emulate realistic client data heterogeneity, each dataset is partitioned
across clients with a Dirichlet distribution $\operatorname{Dir}(\alpha)$ over
the label space, where a smaller concentration parameter $\alpha$ yields a more
skewed (non-IID) partition. We report results under a moderate heterogeneity level
$\operatorname{Dir}(0.1)$ and an extreme level $\operatorname{Dir}(0.02)$, and
evaluate global test accuracy at the final communication round.

\paragraph{Commonsense reasoning.}
\label{appendix:datasets-reasoning}
For the additional evaluation in Table~\ref{tab:reasoning-benchmarks}, we adopt
the Commonsense-170K training and evaluation protocol used by
FedEx-LoRA~\citep{singhal-etal-2025-fedex}, following
\citet{hu2023llmadapters}:
models are fine-tuned on the Commonsense-170K instruction-tuning corpus and
evaluated on eight multiple-choice reasoning tasks. Each task is cast as a
multiple-choice problem in which the model selects the correct answer from the
provided options, and we report per-task accuracy together with the eight-task
average. The benchmarks are:
\begin{itemize}
  \item \textbf{BoolQ}~\citep{clark2019boolq}: naturally occurring yes/no
    questions, each paired with a Wikipedia passage that provides the context.
  \item \textbf{PIQA}~\citep{bisk2020piqa}: physical commonsense questions that
    require choosing the more sensible of two candidate solutions to an everyday
    physical goal.
  \item \textbf{SIQA} (Social IQa)~\citep{sap2019socialiqa}: questions probing
    reasoning about people's social interactions and their likely motivations
    and consequences.
  \item \textbf{HellaSwag}~\citep{zellers2019hellaswag}: sentence-completion
    questions in which the model selects the most plausible continuation of a
    given context.
  \item \textbf{WinoGrande}~\citep{sakaguchi2021winogrande}: a large-scale
    adversarial Winograd schema benchmark requiring binary pronoun/coreference
    resolution.
  \item \textbf{ARC-Easy (ARC-e)} and \textbf{ARC-Challenge
    (ARC-c)}~\citep{clark2018arc}: grade-school science questions split into an
    easier subset and a harder subset of questions that simple retrieval methods
    answer incorrectly.
  \item \textbf{OpenBookQA (OBQA)}~\citep{mihaylov2018openbookqa}: elementary
    science questions that require combining a small ``open book'' of core
    science facts with broad commonsense knowledge.
\end{itemize}
Because our reasoning evaluation follows the same benchmark protocol as
FedEx-LoRA, the results in
Table~\ref{tab:reasoning-benchmarks} are directly comparable to theirs.

\subsection{Baseline Methods}
\label{app:baselines}

We compare against all methods reported in Table~\ref{performance-comparison}.
\textbf{FedIT}~\citep{10447454} applies FedAvg independently to the two LoRA
factors, while \textbf{LoRA-$A^2$}~\citep{koo2025towards} alternates which
LoRA factor is frozen and adaptively selects the rank to improve robustness
under heterogeneous data. \textbf{FLoRA}~\citep{wang2024florafederatedfinetuninglarge}
represents the exact weighted update by stacking the client factors and
broadcasts the resulting high-rank adapter. \textbf{FedEx-LoRA}~\citep{singhal-etal-2025-fedex}
provides an alternative exact aggregation method for federated LoRA.
\textbf{FlexLoRA}~\citep{bai2024federated} materializes the dense aggregate and
applies truncated SVD to construct compact global adapters. \textbf{FLoRG}~\citep{meng2026florg}
aggregates low-rank Gram matrices and uses Procrustes alignment, whereas
\textbf{FLoRIST}~\citep{ramesh2026florist} recompresses the stacked factors
through rectangular SVDs and a small intermediate decomposition.

\begin{table*}[t]
  \centering
  \small
  \caption{Hyperparameter details for the RoBERTa classification and
  LLaMA reasoning experiments. Table~\ref{performance-comparison} uses a
  uniform local LoRA rank of $8$; the heterogeneous-rank configuration in
  Table~\ref{tab:hetero-rank} is specified in a separate row.}
  \label{tab:hyperparameters}
  \begin{tabular}{lcc}
    \toprule
    \textbf{Hyperparameter}
      & \textbf{RoBERTa-base}
      & \textbf{LLaMA-3.2-3B} \\
    \midrule
    Training datasets
      & \shortstack{BANKING77, 20Newsgroups,\\CLINC150, HWU64}
      & COMMONSENSE170K \\
    Optimizer
      & AdamW
      & AdamW \\
    Learning rate
      & $5 \times 10^{-4}$
      & $3 \times 10^{-4}$ \\
    Learning-rate schedule
      & Constant
      & Constant \\
    Warmup ratio
      & $0$
      & $0$ \\
    Batch size (per client)
      & $32$
      & $16$ \\
    Gradient accumulation steps
      & $1$
      & $1$ \\
    Local optimizer steps per round
      & $200$
      & $100$ \\
    Global communication rounds
      & $30$
      & $5$ \\
    Total clients
      & $10$
      & $8$ \\
    Client participation rate
      & $1.0$
      & $1.0$ \\
    Data heterogeneity
      & Dirichlet
      & Pathological non-IID \\
    Dirichlet concentration $\alpha$
      & $0.1,\ 0.02$
      & N/A \\
    LoRA rank (Table~\ref{performance-comparison})
      & $8$
      & $8$ \\
    LoRA ranks (Table~\ref{tab:hetero-rank})
      & N/A
      & $[4,16,32,64,4,16,32,64]$ \\
    LoRA scale $s_k=\alpha_{\mathrm{LoRA},k}/r_k$
      & $2$
      & $2$ \\
    LoRA alpha $\alpha_{\mathrm{LoRA},k}$ (Table~\ref{performance-comparison})
      & $16$
      & $16$ \\
    LoRA alphas $\alpha_{\mathrm{LoRA},k}$ (Table~\ref{tab:hetero-rank})
      & N/A
      & $[8,32,64,128,8,32,64,128]$ \\
    LoRA dropout
      & $0$
      & $0$ \\
    Target modules
      & \shortstack{\texttt{query}, \texttt{key},\\\texttt{value}}
      & \shortstack{\texttt{q\_proj}, \texttt{k\_proj},
                    \texttt{v\_proj}, \texttt{o\_proj},\\
                    \texttt{gate\_proj}, \texttt{up\_proj},
                    \texttt{down\_proj}} \\
    Maximum sequence length
      & $256$
      & $256$ \\
    Training precision
      & FP32
      & BF16 \\
    Random seed
      & $42$
      & $42$ \\
    Energy-retention threshold $\tau$
      & $0.80,\ 0.95$
      & $0.80,\ 0.95$ \\
    Evaluation interval
      & Every $5$ rounds
      & Post-training evaluation \\
    Classifier-head freezing
      & After round $5$
      & N/A \\
    \bottomrule
  \end{tabular}
\end{table*}

\subsection{Hyperparameter Settings}
\label{app:hyperparameters}
Table~\ref{tab:hyperparameters} lists the training and federated learning
hyperparameters for the RoBERTa-base classification experiments and the
LLaMA-3.2-3B reasoning experiments. Unless otherwise noted, we use the same
setting for all methods and datasets to ensure a controlled comparison.

\FloatBarrier

\section{Alternative Orthogonalization Backends}
\label{app:alternative-backends}

Proposition~\ref{prop:core-equivalence} establishes exact coordinate-space
equivalence whenever $\mat{Q}$ is an orthonormal basis whose span contains the
column space of the oriented left factor $\mat{X}$. However, this proposition
does not prescribe how such a basis must be constructed. Although \ours{} uses
reduced QR as its default implementation, other exact orthogonalization methods
can satisfy the same conditions. Specifically, for any $\mat{Q}$ satisfying
\[
  \mat{Q}^{\top}\mat{Q}=\mat{I},
  \qquad
  \operatorname{col}(\mat{X})\subseteq\operatorname{col}(\mat{Q}),
\]
we may define
\[
  \mat{C}=\mat{Q}^{\top}\mat{X},
  \qquad
  \mat{H}=\mat{C}\mat{Y}.
\]
Then $\mat{X}=\mat{Q}\mat{C}$ and
$\mat{X}\mat{Y}=\mat{Q}\mat{H}$. By
Proposition~\ref{prop:core-equivalence}, the compact eigendecomposition
\[
  \mat{H}\mat{H}^{\top}
  =\mat{U}\mat{\Lambda}\mat{U}^{\top}
\]
recovers the nonzero singular spectrum and left singular subspace needed by
\ours{}. Reduced QR is one realization of this construction: if
$\mat{X}=\mat{Q}_{\mathrm{qr}}\mat{R}$, then $\mat{C}=\mat{R}$. The triangular
structure of $\mat{R}$ is incidental; orthonormality and column-space coverage
are the essential conditions.

\paragraph{Polar factorization.}
For full-column-rank $\mat{X}$, the polar decomposition gives
\[
  \begin{aligned}
    \mat{X}&=\mat{Q}_{\mathrm{pol}}\mat{S}, \\
    \mat{Q}_{\mathrm{pol}}
      &=\mat{X}(\mat{X}^{\top}\mat{X})^{-1/2}, \\
    \mat{S}&=(\mat{X}^{\top}\mat{X})^{1/2}.
  \end{aligned}
\]
Here $\mat{Q}_{\mathrm{pol}}$ has orthonormal columns and $\mat{S}$ is symmetric
positive definite. Thus
$\mat{X}\mat{Y}=\mat{Q}_{\mathrm{pol}}(\mat{S}\mat{Y})$, and $\mat{S}$ can
replace the QR factor $\mat{R}$ without changing the subsequent
coordinate-space decomposition. Hence exact \ours{} depends on an orthogonal-basis
factorization, rather than on QR specifically.

\paragraph{Iterative orthogonalization.}
An iterative polar method, such as Newton--Schulz orthogonalization, can instead
produce an approximately orthogonal basis $\widetilde{\mat{Q}}$. The direct drop-in
construction is
\[
  \begin{aligned}
    \widetilde{\mat{C}}&=\widetilde{\mat{Q}}^{\top}\mat{X},
    &\widetilde{\mat{H}}&=\widetilde{\mat{C}}\mat{Y},\\
    \widetilde{\mat{H}}\widetilde{\mat{H}}^{\top}
      &=\widetilde{\mat{U}}\widetilde{\mat{\Lambda}}
        \widetilde{\mat{U}}^{\top}.&&
  \end{aligned}
\]
No triangular structure is required of $\widetilde{\mat{C}}$. With a finite
number of iterations, however,
$\widetilde{\mat{Q}}^{\top}\widetilde{\mat{Q}}$ generally differs from
$\mat{I}$, and therefore
\[
  \widetilde{\mat{Q}}\widetilde{\mat{C}}
  =\widetilde{\mat{Q}}\widetilde{\mat{Q}}^{\top}\mat{X}
  \ne \mat{X}.
\]
This backend consequently introduces an orthogonalization error in addition to
the intended rank-truncation error.

For more accurate reconstruction, one may instead use the least-squares
coordinates
\[
  \widetilde{\mat{C}}
  =(\widetilde{\mat{Q}}^{\top}\widetilde{\mat{Q}})^{-1}
    \widetilde{\mat{Q}}^{\top}\mat{X}.
\]
Then
$\widetilde{\mat{Q}}\widetilde{\mat{C}}
=\mat{P}_{\widetilde{\mat{Q}}}\mat{X}$, which equals $\mat{X}$ whenever
$\widetilde{\mat{Q}}$ spans the same column space as $\mat{X}$. Importantly, exact
reconstruction alone does not make the usual eigendecomposition of
$\widetilde{\mat{H}}\widetilde{\mat{H}}^{\top}$ exact: when
$\widetilde{\mat{Q}}$ is not orthonormal, left multiplication by
$\widetilde{\mat{Q}}$ changes the metric. An exact small coordinate-space variant can
restore the standard form by factoring
\[
  \widetilde{\mat{Q}}^{\top}\widetilde{\mat{Q}}
    =\mat{L}\mat{L}^{\top}, \qquad
  \mat{X}\mat{Y}
    =(\widetilde{\mat{Q}}\mat{L}^{-\top})
     (\mat{L}^{\top}\widetilde{\mat{C}}\mat{Y}),
\]
where $\widetilde{\mat{Q}}\mat{L}^{-\top}$ has orthonormal columns. The
\ours{} coordinate-space decomposition can then be applied to
$\mat{L}^{\top}\widetilde{\mat{C}}\mat{Y}$. This only requires a small
$r\times r$ factorization, but it reduces the advantage of a pure GEMM-based
backend.

\paragraph{Accuracy and performance considerations.}
An iterative method run sufficiently close to convergence can replace thin QR
nearly transparently. In contrast, a fixed-step, low-precision scheme designed
primarily for optimizer updates need not deliver the accuracy required for
recompression. We therefore recommend measuring
\[
  \|\widetilde{\mat{Q}}^{\top}\widetilde{\mat{Q}}-\mat{I}\|_F,
  \qquad
  \frac{\|\mat{X}-\widetilde{\mat{Q}}
    (\widetilde{\mat{Q}}^{\top}\mat{X})\|_F}{\|\mat{X}\|_F},
\]
as well as the gap between the final truncated update and exact QR-based
\ours{}. Whether several Newton--Schulz GEMMs outperform a vendor-optimized
thin QR depends on the tall-skinny dimensions, precision, and GPU. The
algorithmic substitution is therefore valid, while its practical benefit must
be established by an accuracy--latency benchmark on the target hardware.

\end{document}